\documentclass[letterpaper, 10 pt, conference]{ieeeconf}  

\IEEEoverridecommandlockouts                              

\usepackage[left=1.88cm,right=1.88cm,top=1.88cm,bottom=1.88cm]{geometry}
\usepackage{amsmath,amssymb,amsfonts}
\usepackage{graphicx}
\usepackage{textcomp}
\usepackage{xcolor}
\usepackage{bbm}
\let\labelindent\relax
\usepackage{enumitem}
\usepackage{tabularx}

\usepackage{times}
\usepackage{multicol}

\usepackage{graphicx}
\usepackage{bm}
\usepackage[nolist]{acronym}

\usepackage[linesnumbered,ruled]{algorithm2e}

\usepackage{mathtools}
\usepackage{array}
\usepackage{dblfloatfix}
\usepackage{diagbox}
\usepackage{etoolbox}\AtBeginEnvironment{algorithmic}{\small}
\usepackage{color,soul} 

\usepackage{lipsum}
\usepackage[colorinlistoftodos]{todonotes}

\usepackage{amsthm}

\usepackage{bbm} 
\usepackage[colorlinks = True, linkcolor = blue, citecolor = blue]{hyperref}

\DeclareMathOperator*{\argmin}{arg\,min}

\usepackage{cleveref}
\Crefname{problem}{Problem}{Problems}
\Crefname{thm}{Theorem}{Theorems}
\Crefname{prop}{Proposition}{Propositions}
\Crefname{assum}{Assumption}{Assumptions}
\Crefname{lem}{Lemma}{Lemmas}
\Crefname{enumi}{Assumption}{Assumptions}
\usepackage{rotating}
\usepackage{multirow}
\usepackage{caption}
\usepackage[
    style=ieee,
    doi=false,
    isbn=false,
    url=false,
    eprint=false,
    backend=biber,
    natbib=true
    ]{biblatex}

\bibliography{refs}

\title{\LARGE \bf
Exact, Efficient, and Safe Occlusion-Aware Planning\\
Using AH-Polyhedrons
}

\author{Long Kiu Chung$^{1,2}$, David Isele$^1$, Toktam Mohammadnejad$^1$, Faizan M. Tariq$^1$,\\
Sangjae Bae$^1$, Shreyas Kousik$^2$, Jovin D'sa$^1$
\thanks{
$^1$Honda Research Institute (HRI), Mountain View, CA.
$^2$ Georgia Institute of Technology, Atlanta, GA.
Corresponding Authors: \href{mailto:lchung33@gatech.edu}{\texttt{lchung33@gatech.edu}}, \href{mailto:jovin_dsa@honda-ri.com}{\texttt{jovin\_dsa@honda-ri.com}}. 
All work was done when Long Kiu Chung was employed by HRI.
}}

\begin{document}

\newif\ifshowcomments
 \showcommentstrue  

\ifshowcomments
    \newcommand{\bae}[1]{\hl{[SB: #1]}\protect\color{black}} 
    \newcommand{\di}[1]{\hl{[DI: #1]}\protect\color{black}} 
    \newcommand{\ft}[1]{\hl{[FT: #1]}\protect\color{black}} 
    \newcommand{\tm}[1]{\hl{[TM: #1]}\protect\color{black}} 
    \newcommand{\jd}[1]{\hl{[JD: #1]}\protect\color{black}} 
    \newcommand{\edgar}[1]{{\color{red}{LKC: #1}}\protect\color{black}}
    \newcommand{\shrey}[1]{\hl{[SK: #1]}\protect\color{black}}
    
\else
    \newcommand{\bae}[1]{}
    \newcommand{\di}[1]{}
    \newcommand{\ft}[1]{}
    \newcommand{\tm}[1]{}
    \newcommand{\jd}[1]{}
    \newcommand{\edgar}[1]{}
    \newcommand{\shrey}[1]{}
\fi

\newtheorem{defn}{Definition}
\newtheorem{rem}[defn]{Remark}
\newtheorem{lem}[defn]{Lemma}
\newtheorem{prop}[defn]{Proposition}
\newtheorem{assum}[defn]{Assumption}
\newtheorem{ex}[defn]{Example}
\newtheorem{runex}[defn]{Running Example}
\newtheorem{thm}[defn]{Theorem}
\newtheorem{cor}[defn]{Corollary}
\newtheorem{problem}[defn]{Problem}
\setlist[enumerate,1]{label=\arabic*),
                      ref=\theassum.\arabic*}

\providecommand{\R}{\ensuremath \mathbb{R}}
\newcommand{\N}{\ensuremath \mathbb{N}}

\newcommand{\regtext}[1]{\mathrm{\textnormal{#1}}}
\newcommand{\vc}[1]{\mathbf{#1}}

\newcommand{\suchthat}{\regtext{ s.t. }}

\newcommand{\tp}{^\intercal}
\newcommand{\norm}[1]{\left\Vert#1\right\Vert}
\newcommand{\abs}[1]{\left\vert#1\right\vert}

\newcommand{\eye}{\vc{I}}
\newcommand{\Acon}{\vc{A}}
\newcommand{\Ccon}{\vc{C}}
\newcommand{\rotmat}{\vc{R}}

\newcommand{\zeros}{\mathbf{0}}
\newcommand{\ones}{\mathbf{1}}
\newcommand{\bcon}{\vc{b}}
\newcommand{\dcon}{\vc{d}}
\newcommand{\xv}{\vc{x}}
\newcommand{\yv}{\vc{y}}
\newcommand{\constraint}{\vc{g}}
\newcommand{\vertex}{\vc{w}}

\newcommand{\state}{\vc{x}}
\newcommand{\statenow}{\state_{\ts}}
\newcommand{\stateset}{X}
\newcommand{\ctrl}{\vc{u}}
\newcommand{\ctrlmax}{u_{\regtext{max}}}
\newcommand{\ctrlmin}{u_{\regtext{min}}}
\newcommand{\optctrl}{\ctrl^{*}}
\newcommand{\ctrlnow}{\ctrl_{\ts}}
\newcommand{\ctrlset}{U}
\newcommand{\wspace}{\vc{p}}
\newcommand{\wsset}{P}
\newcommand{\nonws}{\vc{v}}
\newcommand{\nonwsset}{V}
\newcommand{\dyn}{\vc{f}}
\newcommand{\ctrlpol}{\vc{\pi}}
\newcommand{\nwspace}{\ndim_{\regtext{p}}}
\newcommand{\nnonws}{\ndim_{\regtext{v}}}
\newcommand{\nstate}{\ndim_{\regtext{x}}}
\newcommand{\nctrl}{\ndim_{\regtext{u}}}
\newcommand{\volume}{V}
\newcommand{\occupancy}{S}
\newcommand{\xs}{x}
\newcommand{\xsnow}{\xs_\ts}
\newcommand{\xsnext}{\xs_{\ts+1}}
\newcommand{\ys}{y}
\newcommand{\ysnow}{\ys_\ts}
\newcommand{\ysnext}{\ys_{\ts+1}}
\newcommand{\vels}{v}
\newcommand{\velsnow}{\vels_\ts}
\newcommand{\velsnext}{\vels_{\ts+1}}
\newcommand{\accels}{a}
\newcommand{\accelsnow}{\accels_\ts}
\newcommand{\adstate}{\tilde{\state}}
\newcommand{\adstatenow}[1]{{\adstate}_{\ts, #1}}
\newcommand{\adstatenext}[1]{{\adstate}_{\ts + 1, #1}}
\newcommand{\adstateset}{\tilde{\stateset}}
\newcommand{\adwsset}{\tilde{\wsset}}
\newcommand{\adnonws}{\tilde{\nonws}}
\newcommand{\adnonwsset}{\tilde{\nonwsset}}
\newcommand{\adctrl}{\tilde{\ctrl}}
\newcommand{\adctrlnow}[1]{{\adctrl}_{\ts, #1}}
\newcommand{\adctrlnext}[1]{{\adctrl}_{\ts + 1, #1}}
\newcommand{\adctrlset}{\tilde{\ctrlset}}
\newcommand{\adwspace}{\tilde{\wspace}}
\newcommand{\addyn}{\tilde{\dyn}}
\newcommand{\nadstate}[1]{\ndim_{\tilde{\regtext{x}}, #1}}
\newcommand{\nadctrl}[1]{\ndim_{\tilde{\regtext{u}}, #1}}
\newcommand{\nadnonws}[1]{\ndim_{\tilde{\regtext{v}}, #1}}
\newcommand{\advolume}{\tilde{\volume}}
\newcommand{\adoccupancy}{\tilde{\occupancy}}
\newcommand{\adGdyn}{\tilde{\vc{G}}}
\newcommand{\adGdynstate}[1]{\adGdyn_{\regtext{x}, #1}}
\newcommand{\adGdynctrl}[1]{\adGdyn_{\regtext{u}, #1}}
\newcommand{\adhdyn}{\tilde{\vc{h}}}
\newcommand{\adxs}{\tilde{x}}
\newcommand{\adxsnow}[1]{\adxs_{\ts, #1}}
\newcommand{\adxsnext}[1]{\adxs_{\ts+1, #1}}
\newcommand{\adys}{\tilde{y}}
\newcommand{\adysnow}[1]{\adys_{\ts, #1}}
\newcommand{\adysnext}[1]{\adys_{\ts+1, #1}}
\newcommand{\advelsnow}[1]{\tilde{v}_{\ts, #1}}
\newcommand{\advelsnext}[1]{\tilde{v}_{\ts + 1, #1}}
\newcommand{\advelxsnow}[1]{\tilde{v}_{\regtext{x}, {\ts, #1}}}
\newcommand{\advelxsnext}[1]{\tilde{v}_{\regtext{x}, {\ts+1, #1}}}
\newcommand{\advelysnow}[1]{\tilde{v}_{\regtext{y}, {\ts, #1}}}
\newcommand{\advelysnext}[1]{\tilde{v}_{\regtext{y}, {\ts+1, #1}}}
\newcommand{\adaccelsnow}[1]{\tilde{a}_{\ts, #1}}
\newcommand{\adaccelsnext}[1]{\tilde{a}_{\ts + 1, #1}}
\newcommand{\adaccelxsnow}[1]{\tilde{a}_{\regtext{x}, {\ts, #1}}}
\newcommand{\adaccelxsnext}[1]{\tilde{a}_{\regtext{x}, {\ts+1, #1}}}
\newcommand{\adaccelysnow}[1]{\tilde{a}_{\regtext{y}, {\ts, #1}}}
\newcommand{\adaccelysnext}[1]{\tilde{a}_{\regtext{y}, {\ts+1, #1}}}

\newcommand{\bustate}{\hat{\state}}
\newcommand{\bustatenow}{\bustate_{\ts}}
\newcommand{\bustatenext}{\bustate_{\ts + 1}}
\newcommand{\buctrl}{\hat{\ctrl}}
\newcommand{\buctrlpol}{\hat{\ctrlpol}}
\newcommand{\buoccupancy}{\hat{\occupancy}}

\newcommand{\Aconstate}{\Acon_{\regtext{x}}}
\newcommand{\bconstate}{\bcon_{\regtext{x}}}
\newcommand{\Aconctrl}{\Acon_{\regtext{u}}}
\newcommand{\bconctrl}{\bcon_{\regtext{u}}}
\newcommand{\Aconadstate}[1]{\tilde{\Acon}_{\regtext{x}, #1}}
\newcommand{\bconadstate}[1]{\tilde{\bcon}_{\regtext{x}, #1}}
\newcommand{\Aconadctrl}[1]{\tilde{\Acon}_{\regtext{u}, #1}}
\newcommand{\bconadctrl}[1]{\tilde{\bcon}_{\regtext{u}, #1}}
\newcommand{\Aconvol}{\Acon_\regtext{\occupancy}}
\newcommand{\bconvol}{\bcon_\regtext{\occupancy}}
\newcommand{\Aconadvol}[1]{{\Acon}_{\tilde{\regtext{\occupancy}}, #1}}
\newcommand{\Bconadvol}[1]{{\vc{B}}_{\tilde{\regtext{\occupancy}}, #1}}
\newcommand{\bconadvol}[1]{{\vc{b}}_{\tilde{\regtext{\occupancy}}, #1}}
\newcommand{\nconvol}{\ndim_{\regtext{c},\regtext{\occupancy}}}
\newcommand{\nconadvol}[1]{\ndim_{\regtext{c}, \tilde{\regtext{\occupancy}}, #1}}
\newcommand{\Aconhset}[1]{\Acon_{\regtext{\hset}, #1}}
\newcommand{\bconhset}[1]{\bcon_{\regtext{\hset}, #1}}
\newcommand{\Cconhset}[1]{\Ccon_{\regtext{\hset}, #1}}
\newcommand{\dconhset}[1]{\dcon_{\regtext{\hset}, #1}}
\newcommand{\nconhset}[1]{\ndim_{\regtext{c}, \regtext{\hset}, #1}}
\newcommand{\ngenhset}[1]{\ndim_{\regtext{g}, \regtext{\hset}, #1}}
\newcommand{\Aconfhset}[1]{\Acon_{\regtext{\fhset}, #1}}
\newcommand{\bconfhset}[1]{\bcon_{\regtext{\fhset}, #1}}
\newcommand{\Cconfhset}[1]{\Ccon_{\regtext{\fhset}, #1}}
\newcommand{\dconfhset}[1]{\dcon_{\regtext{\fhset}, #1}}
\newcommand{\Acondz}[1]{\Acon_{\regtext{\dz}, #1}}
\newcommand{\bcondz}[1]{\bcon_{\regtext{\dz}, #1}}
\newcommand{\Ccondz}[1]{\Ccon_{\regtext{\dz}, #1}}
\newcommand{\dcondz}[1]{\dcon_{\regtext{\dz}, #1}}
\newcommand{\Aconstop}[1]{\Acon_{\regtext{stop}, #1}}
\newcommand{\bconstop}[1]{\bcon_{\regtext{stop}, #1}}
\newcommand{\Aconsafe}[1]{\Acon_{\regtext{safe}, #1}}
\newcommand{\bconsafe}[1]{\bcon_{\regtext{safe}, #1}}
\newcommand{\Aconobs}[1]{\Acon_{\regtext{\obs}, #1}}
\newcommand{\bconobs}[1]{\bcon_{\regtext{\obs}, #1}}
\newcommand{\Aconminus}{\Acon_{3}}
\newcommand{\bconminus}{\bcon_{3}}
\newcommand{\Cconminus}{\Ccon_{3}}
\newcommand{\dconminus}{\dcon_{3}}
\newcommand{\nconi}[1]{\ndim_{\regtext{c}, #1}}
\newcommand{\Lcon}{\vc{L}}

\newcommand{\braketime}{\ts_{\regtext{brake}}}
\newcommand{\comptime}{\ts_{\regtext{comp}}}
\newcommand{\stepstogoal}{\ndim_{\regtext{goal}}}
\newcommand{\safetyrate}{r_{\regtext{safe}}}
\newcommand{\velbrake}{\vels_{\regtext{brake}}}

\newcommand{\ts}{t}
\newcommand{\is}{i}
\newcommand{\js}{j}
\newcommand{\ks}{k}
\newcommand{\ms}{m}
\newcommand{\maxdist}{\overline{p}}
\newcommand{\gapsize}{g}
\newcommand{\dt}{\Delta\ts}
\newcommand{\heading}{\theta}
\newcommand{\taus}{\tau}
\newcommand{\lambdas}{\lambda}
\newcommand{\length}{\ell}
\newcommand{\width}{w}
\newcommand{\adlength}{\tilde{\length}}
\newcommand{\adwidth}{\tilde{\width}}
\newcommand{\adheading}{\tilde{\heading}}
\newcommand{\cost}{c}
\newcommand{\eps}{\epsilon}

\newcommand{\ndim}{n}
\newcommand{\mdim}{m}
\newcommand{\ncon}{\ndim_{\regtext{c}}}
\newcommand{\ngen}{\ndim_{\regtext{g}}}
\newcommand{\nstop}{\ndim_{\regtext{stop}}}
\newcommand{\ntype}{\ndim_{\regtext{type}}}
\newcommand{\ntime}{\ndim_{\ts}}
\newcommand{\nhset}[1]{\ndim_{\regtext{\hset}, #1}}
\newcommand{\ncost}{\ndim_{\regtext{cost}}}
\newcommand{\nconstraint}{\ndim_{\regtext{con}}}
\newcommand{\nbisect}{\ndim_{\regtext{bi}}}
\newcommand{\nobs}{\ndim_{\regtext{obs}}}
\newcommand{\nvert}[1]{\ndim_{\regtext{vert}, #1}}
\newcommand{\nregion}[1]{\ndim_{\regtext{\appearance}, #1}}
\newcommand{\ngrid}{\ndim_{\regtext{grid}}}

\newcommand{\hpoly}{\mathcal{H}}
\newcommand{\ahpoly}{\mathcal{AH}}

\newcommand{\fhset}{F}
\newcommand{\dz}{D}
\newcommand{\hset}{H}
\newcommand{\obs}{O}
\newcommand{\occlusion}{A}
\newcommand{\appearance}{R}
\newcommand{\candidate}{C}
\newcommand{\polyhedron}{Q}

\newcommand{\ctrllb}{\underline{u}}
\newcommand{\ctrlub}{\overline{u}}
\newcommand{\ctrlmid}{u}
\newcommand{\ctrlsafelb}{\underline{u}_{\regtext{safe}}}
\newcommand{\ctrlsafeub}{\overline{u}_{\regtext{safe}}}

\newcommand{\Aconocc}[1]{\Acon_{\regtext{occ}, #1}}
\newcommand{\bconocc}[1]{\bcon_{\regtext{occ}, #1}}
\newcommand{\Aconvis}{\Acon_{\regtext{vis}}}
\newcommand{\bconvis}{\bcon_{\regtext{vis}}}
\newcommand{\Aconmax}{\overline{\Acon}}
\newcommand{\Aconmin}{\underline{\Acon}}
\newcommand{\bconmax}{\overline{\bcon}}
\newcommand{\bconmin}{\underline{\bcon}}
\newcommand{\ray}{\vc{r}}
\newcommand{\raymin}{\underline{\ray}}
\newcommand{\raymax}{\overline{\ray}}
\newcommand{\jsmin}{\underline{\js}}
\newcommand{\jsmax}{\overline{\js}}
\newcommand{\thetas}{\theta}
\newcommand{\visidx}{K_\regtext{vis}}
\newcommand{\phis}{\phi}
\newcommand{\psis}{\psi}
\newcommand{\av}{\vc{a}}

\newcommand{\mergecost}{\cost_{\regtext{merge}}}

\maketitle


\begin{abstract}
Safely handling occlusions is a fundamental challenge for autonomous mobile robots operating in dynamic environments.
This issue is especially prominent in autonomous valet parking (AVP), where traffic rules are lax, occlusions are frequent and cluttered, and overly conservative behavior can leave vehicles stuck.
However, existing methods either lack formal safety guarantees, assume agents follow road structures, or introduce conservatism, leaving occlusion-aware planning for AVP an open challenge.
In this paper, we propose APRO (AH-Polyhedron Reachability for Occlusions), an exact and efficient occlusion-aware planning framework based on game-theoretic active perception and AH-polyhedron reachability analysis with AVP as our canonical use case.
Our key insight is to reformulate set-based safety conditions in prior work as unions of AH-polyhedrons, enabling \textit{exact} safety verification through linear programming (LP) without \textit{any} additional conservatism in set computations or assumptions on road topology.
We further show how the resulting safety conditions can be integrated into optimization-based planners or a bisection search scheme for real-time applications.
We validate our method in simulation and hardware experiments, including data replay on a real-world parking lot dataset.
Experimental results demonstrate that our method consistently achieved a 100\% safety rate across all evaluated scenarios while maintaining real-time performance, resulting in safer and more optimal decisions than existing methods with formal safety guarantees.
Website: \href{https://sites.google.com/view/chung2026exact}{https://sites.google.com/view/chung2026exact}.

\end{abstract}
\section{Introduction}

In practice, deployed systems rely on sensors that are susceptible to occlusion.
Thus, in addition to observed obstacles and agents, robots should also plan for hidden agents in occluded areas.
However, existing occlusion-aware planners either do not provide safety guarantees \cite{bouton2018scalable,hubmann2019pomdp,yu2019occlusion,yu2020risk,park2023occlusion,kocc2021pedestrian,van2023overcoming,gilhuly2022looking}, assume agents follow road topologies \cite{zhang2021safe, park2023occlusion, koschi2020set, orzechowski2018tackling}, or are overly conservative \cite{firoozi2022occlusion, orzechowski2018tackling, lee2021limited, koschi2020set}.
In contrast, many scenarios involve agents (e.g.\ pedestrians) that deviate from the road structure and tight, cluttered spaces with frequent occlusions, which make deployment of existing methods difficult.
In this paper, we present APRO (AH-Polyhedron Reachability for Occlusions), an occlusion-aware planning framework based on game theory and exact AH-polyhedron (affine transformation of a polyhedron in halfspace representation) reachability analysis that provides safety guarantees without \textit{any} conservativeness in set computation or assumptions on road topologies.
As a practical instantiation, we showcase APRO on autonomous valet parking (AVP) systems, in which traffic rules are lax (agents do not strictly follow road structures), traveling too slow and frequent stopping is undesirable (leaves little room for conservatism), and accidents from cluttered, tight occlusions are frequent \cite{lee2021limited} (safety guarantees are important).
An overview of APRO is shown in \Cref{fig:front_figure}.

\begin{figure}[t]
\centering
    \includegraphics[width=1\columnwidth]{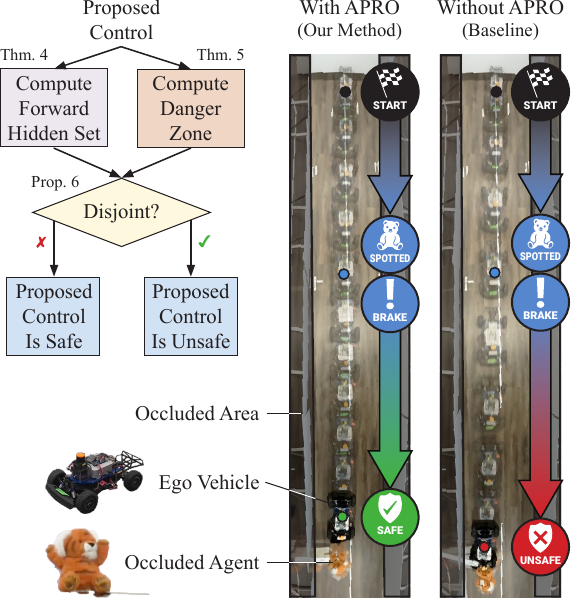}
\caption{
(Top Left) An overview of our method, APRO, with relevant sections labeled.
(Right) The setup for our hardware experiment in \Cref{sec:hardware}.
APRO makes the ego move more cautiously around occluded area in anticipation of emergency braking for occluded agents.}\label{fig:front_figure}
\vspace*{-0.5cm}
\end{figure}

\subsection{Related Work}
We divide the existing literature on occlusion-aware planning into those with and without formal guarantees, the former of which is further categorized into set-based and value-function-based reachability methods.
APRO is a set-based reachability method with formal guarantees.

\subsubsection{Non-Formal Methods}
Approaches without formal guarantees usually handle occlusions by simulating phantom agents and/or by reasoning with probabilities.
For example, some methods use reachability analysis to determine the future distribution of phantom objects, obtained from sampling \cite{yu2019occlusion, yu2020risk} or by intersecting the road topology \cite{park2023occlusion} with the occluded area.
Other methods deter plans in high-risk areas by estimating pedestrian emergence probabilities \cite{kocc2021pedestrian}, computing risk fields \cite{van2023overcoming}, or by maximizing information gain \cite{gilhuly2022looking}.
There are also methods that formulate occlusion-aware planning as a partially observable Markov decision process (POMDP) by integrating an observation model in the belief state \cite{bouton2018scalable, hubmann2019pomdp}.
While these methods are generally faster and/or more versatile, they can still result in collisions, whereas formal methods such as ours do not (under the same modeling assumptions).

\subsubsection{Set-Based Reachability Methods}
Most approaches with formal guarantees are based on set-based reachability analysis that computes the forward reachable set (FRS), which encompasses the adversaries' possible future states.
For example, overapproximative FRS for double integrators (based on Kamm's circle \cite{rajamani2006vehicle}) has been applied to occlusion-aware planners for structured \cite{orzechowski2018tackling, koschi2020set} and unstructured \cite{lee2021limited} road networks.
Similarly, Occlusion-Aware Model Predictive Control (OA-MPC) \cite{firoozi2022occlusion} computes the FRS of adversaries as capsules based on their max speed alone, which can be embedded into quadratic programs (QPs) by overapproximating the ego's volume with a circle.
While these methods are guaranteed to be safe, they are generally more restrictive on the dynamics of the adversaries.
Moreover, as shown in \Cref{fig:narrow_gap} and \Cref{sec:sim_exp}, they all incur overapproximations that can leave the ego stuck.
In constrast, APRO is \textit{exact} and does not compromise between safety and performance.

\subsubsection{Value-Function-Based Reachability Methods}
Instead of computing reachable sets geometrically, methods such as Hamilton-Jacobi (HJ) reachability represents the reachable sets as sub-level sets of value functions.
Our method is mainly inspired by \cite{zhang2021safe}, which proposed mathematical set formulations from game theory that implies safety from occluded agents when the sets are disjoint.
However, their implementation assumes the ego and the adversary follow a road network, with which these sets can be computed trivially along the lanes.
For more general scenarios, they suggested computing the sets with HJ reachability analysis.
That said, HJ reachability is known to be slow, scale poorly with dimensions, and can suffer from numerical inaccuracies that forfeit the safety guarantees \cite{he2025threshold}.
While learning-based HJ methods \cite{fisac2019bridging, bansal2021deepreach} have been proposed to overcome the issues in scalability and speed, approximation errors from neural networks can further loosen the guarantees \cite{yang2025scalable}.
Instead, we propose to reformulate and compute the sets in \cite{zhang2021safe} as AH-polyhedrons, which is faster, safer, more scalable, and exact. 

\subsection{Contributions}
\begin{enumerate}
    \item We propose APRO, a method to compute the occlusion-aware safety conditions in \cite{zhang2021safe} without assuming road structures or incurring conservatism in set computation using AH-polyhedrons, which was shown to be fast, less conservative, and safer compared to existing methods \cite{zhang2021safe, lee2021limited, firoozi2022occlusion} in simulation experiments.
    \item We develop practical control schemes to implement APRO on, including nonlinear optimization and real-time bisection search, enabling integration with existing optimization- and search-based planners.
    \item We demonstrate the praticality of APRO by integrating it on a Hybrid A* AVP planner \cite{chung2026selecting, nawaz2025occupancy} navigating a real-world dataset replay \cite{shen2022parkpredict+} and by deploying it on a F1/10 class robotic vehicle \cite{o2019f1} equipped with lidar sensors.
    APRO maintains a $100\%$ safety rate across all experiments.
\end{enumerate}

\section{Preliminaries}
We now introduce our notation for AH-polyhedrons and succinctly summarize the results in \cite{zhang2021safe}.

\subsection{AH-polyhedrons}
An $\ndim$-dimensional AH-polyhedron $\ahpoly(\Acon, \bcon, \Ccon, \dcon) \subseteq \R^\ndim$ is an affine transformation of an $\ngen$-dimensional polyhedron, parameterized by constraints $\Acon \in \R^{\ncon\times\ngen}, \bcon \in \R^{\ncon}$ and affine transformation $\Ccon \in \R^{\ndim\times\ngen}, \dcon \in \R^{\ndim}$ as \cite{sadraddini2019linear}:
\begin{align}
    \ahpoly(\Acon,\bcon,\Ccon,\dcon) = \{\Ccon\xv+\dcon\ |\ \Acon\xv\leq\bcon\}.
\end{align}
An H-polyhedron $\hpoly(\Acon, \bcon)$ is a special case of AH-polyhedron without affine transformation $\ahpoly(\Acon, \bcon, \eye, \zeros)$.
Bounded H-polyhedrons and AH-polyhedrons are referred to as H-polytopes and AH-polytopes, respectively.

To compute occlusion-aware safety conditions, we make use of the closed-form intersection $\cap$, Cartesian product $\times$, and emptiness check \cite{sadraddini2019linear, herceg2013multi} of H and AH-polyhedrons.
Specifically, an AH-polyhedron $\ahpoly(\Acon, \bcon, \Ccon, \dcon)$ is empty \textit{iff} the linear program (LP) $\min\{\zeros\ |\ \Acon \xv \leq \bcon\}$ is infeasible \cite{chung2025guaranteed}.

\subsection{Game-Theoretic Active Perception}

It is proven in \cite{zhang2021safe} that the ego's current control signal $\ctrlnow$ is safe from an occluded adversarial agent of type $\is$ if the forward hidden set is disjoint from the danger zone, which we formally define below.

The forward hidden set $\fhset_\is \subseteq \R^{\nadstate{\is}}$ is the set of possible adversary states at the next timestep, assuming it is currently being occluded.
Mathematically,
\begin{align}\label{eq:fhset_def}
\begin{split}
    \fhset_\is =& \{\addyn_\is(\adstatenow{\is},\adctrlnow{\is})\ |\ \adstatenow{\is} \in \hset_\is, \adctrlnow{\is} \in \adctrlset_\is\},
\end{split}
\end{align}
where the hidden set $\hset_\is \subseteq \adstateset_\is$ is the set of adversary states where it cannot be detected by the ego at the current timestep, $\addyn_\is: \adstateset_\is\times\adctrlset_\is \to \adstateset_\is$ is the adversary dynamics, and $\adstateset_\is \subseteq \R^{\nadstate{\is}}$ and $\adctrlset_\is \subseteq \R^{\nadctrl{\is}}$ are the adversary's state and control domain.

The danger zone $\dz_\is \subseteq \R^{\nadstate{\is}}$ (also known as capture basin in game theory) is the set of adversary states at the next timestep where there exists a strategy for it to capture the ego before the ego arrives at an invariant safe set.
Mathematically, $\dz_\is = \bigcup_{\js=1}^{\nstop(\ctrlnow)}\dz_{\is,\js}$, where
\begin{align}
    \begin{split}
    \dz_{\is,\js} =& \{\adstatenext{\is}\in \adstateset_\is\ |\ \adoccupancy_{\ts + \js, \is}\cap\buoccupancy_{\ts + \js}\neq\emptyset, \adctrl_{\ts + \ks, \is} \in \adctrlset_\is\\
    &\adstate_{\ts + \ks + 1, \is} = \addyn_\is(\adstate_{\ts + \ks, \is}, \adctrl_{\ts + \ks, \is}),\ks=1, \cdots, \js - 1\},\label{eq:dzone_def}
    \end{split}
\end{align}
where $\nstop:\R^{\nctrl}\to\N$ is the number of timesteps needed for the ego to arrive at an invariant safe set after applying $\ctrlnow$, $\adoccupancy_{\ts + \js, \is} \subseteq\R^{\nwspace}$ is the adversary's occupied space at the $\ts + \js^\regtext{th}$ timestep and a function of the future adversary states $\adstate_{\ts + \js, \is}$, and $\buoccupancy_{\ts + \js} \subseteq\R^{\nwspace}$ is the ego's occupied space at the $\ts + \js^\regtext{th}$ timestep and a function of the future ego states $\bustate_{\ts + \js}$.
In this case, the ego's future states are given by:
\begin{subequations}\label{eq:backup_traj}
\begin{align}
    \bustate_{\ts + \ks + 1} =& \dyn(\bustate_{\ts + \ks}, \buctrl_{\ts + \ks}),\\
    \buctrl_{\ts + \ks} =& \buctrlpol(\bustate_{\ts + \ks}),
\end{align}
\end{subequations}
for $\ks = 1, \cdots, \nstop(\ctrlnow) - 1$, where $\bustatenext = \dyn(\statenow, \ctrlnow)$, $\dyn: \stateset\times\ctrlset\to\stateset$ is the ego dynamics, $\buctrlpol:\stateset\to\ctrlset$ is the backup control policy of the ego, and $\stateset\subseteq\R^{\nstate}$ and $\ctrlset\subseteq\R^{\nctrl}$ are the ego's state and control domains.
The backup control policy guides the ego to an invariant safe set at timestep $\ts + \nstop$ if it were to spot the adversary after applying the control $\ctrlnow$.

If the ego is non-reactive towards the adversary's actions, i.e.\ $\buctrlpol$ is not a function of the adversary's states and control, then the disjoint condition $\fhset_\is \cap \dz_\is = \emptyset$ implies safety for \textit{any} number of adversaries of type $\is$.
A common notion of safety is \textit{passive safety} \cite{koschi2020set,vaskov2019towards}, where the ego is considered safe if it does not collide when it is moving.
In this case, the ego's invariant safe set is defined as the ego being still, and the backup control policy can simply be hard braking.

\section{Problem Formulation}\label{sec:assumptions}
Our overall goal is to derive a control policy $\ctrlnow$ for the ego at the current timestep $\ts$ such that, if the ego were to observe any of the $\ntype$ types of adversaries at the next timestep $\ts + 1$, the ego's backup control policy $\buctrlpol$ is able to guide the ego into an invariant safe set before collision occurs.
We first state assumptions on sets and adversary dynamics, then mathematically describe our problem.


\begin{assum}\label{assum:hpoly_sets}
    All sets are represented as follows:
    \begin{enumerate}
        \item The ego and adversaries' domains are H-polyhedrons $\stateset = \hpoly(\Aconstate, \bconstate), \ctrlset = \hpoly(\Aconctrl, \bconctrl), \adstateset_\is = \hpoly(\Aconadstate{\is}, \bconadstate{\is}), \adctrlset_\is = \hpoly(\Aconadctrl{\is}, \bconadctrl{\is}), \is=1,\cdots,\ntype$.
        \item The ego's occupied space $\buoccupancy_{\ts + \js}$ is an H-polyhedron $\buoccupancy_{\ts + \js} = \hpoly(\Aconvol(\bustate_{\ts + \js}), \bconvol(\bustate_{\ts + \js}))$ for $\js = 1, \cdots, \nstop(\ctrlnow)$, where $\Aconvol: \stateset \to \R^{\nconvol \times \nwspace}$, $\bconvol: \stateset \to \R^{\nconvol}$.
        \item The adversaries' occupied space $\adoccupancy_{\ts + \js, \is}$ is an H-polyhedron $\adoccupancy_{\ts + \js, \is} = \hpoly(\Aconadvol{\is}, \Bconadvol{\is}\adstate_{\ts+\js,\is}+\bconadvol{\is})$
        for $\is=1, \cdots, \ntype, \js = 1, \cdots, \nstop(\ctrlnow)$.
        \item \label{assum:hidden_set} The current hidden set $\hset_\is$ is a union of $\nhset{\is}$ AH-polyhedrons $\hset_\is = \bigcup_{\ks=1}^{\nhset{\is}}\hset_{\is,\ks}$, where $\hset_{\is,\ks}=\ahpoly(\Aconhset{\is,\ks},\bconhset{\is,\ks}, \Cconhset{\is,\ks},\dconhset{\is,\ks})$.
    \end{enumerate}
\end{assum}
\Cref{assum:hpoly_sets} is not overly restrictive, since all convex polyhedrons can be expressed as H-polyhedrons or AH-polyhedrons \cite{sadraddini2019linear}.
If representation of obstacles is given instead of hidden set, \Cref{assum:hidden_set} can be fulfilled by raycasting on the obstacle, then performing Pontryagin difference \cite{herceg2013multi} to account for adversary volume.

\begin{assum}\label{assum:affine_dyn}
    The dynamics of each type of occluded adversary are affine:
    \begin{align}
    \begin{split}
        \adstatenext{\is} = \addyn_\is(\adstatenow{\is},\adctrlnow{\is}) &= \adGdyn_\is \begin{bmatrix}
            \adstatenow{\is}\\
            \adctrlnow{\is}
        \end{bmatrix} + \adhdyn_\is, \\
        &= \adGdynstate{\is}\adstatenow{\is}+\adGdynctrl{\is}\adctrlnow{\is} + \adhdyn_\is,
    \end{split}
    \end{align}
    for all $\adstatenow{\is}, \adstatenext{\is} \in \adstateset_\is, \adctrlnow{\is} \in \adctrlset_\is, \is=1,\cdots,\ntype$, where $\adGdyn_\is=\begin{bmatrix}
        \adGdynstate{\is}\tp & \adGdynctrl{\is}\tp
    \end{bmatrix}\tp\in\R^{\nadstate{\is} \times (\nadstate{\is} + \nadctrl{\is})}, \adGdynstate{\is}\in\R^{\nadstate{\is} \times \nadstate{\is}},\adGdynctrl{\is}\in\R^{\nadstate{\is} \times \nadctrl{\is}}$, and $\adhdyn_\is\in\R^{\nadstate{\is}}$.
\end{assum}
\Cref{assum:affine_dyn} is not overly restrictive either, as it is often sufficient to model occluded adversaries with simple dynamic models such as double integrators \cite{orzechowski2018tackling, koschi2020set, lee2021limited}.
That said, APRO can theoretically be extended to piecewise-affine (PWA) and learned systems by separately evaluating each affine region \cite{chung2024goal, chung2025guaranteed}, which we leave as future work.

\begin{problem}[Occlusion-Aware Safe Control]\label{prob:occlusion_control}
    Given parameters of occupied space $\Aconvol, \bconvol, \Aconadvol{\is}, \Bconadvol{\is}, \bconadvol{\is}$, dynamics $\dyn, \addyn_\is$, backup policy $\buctrlpol$, domain $\stateset, \ctrlset, \adstateset_\is, \adctrlset_\is$, and hidden set $\hset_\is$ for $\is=1,\cdots,\ntype$, find control $\ctrlnow$ such that
    \begin{align}
        \buoccupancy_{\ts + \js} \cap \adoccupancy_{\ts + \js, \is} = \emptyset,
    \end{align}
    for all $\adstatenow{\is}\in\hset_\is,\adctrlnow{\is}, \cdots, \adctrl_{\ts + \js - 1, \is}\in\adctrlset_\is,\is=1,\cdots,\ntype, \js=1,\cdots,\nstop(\ctrlnow)$.
\end{problem}
In a planning framework, occlusion-aware safety can be maintained over a trajectory by repeatedly solving \Cref{prob:occlusion_control} at each timestep.
If no solution exists at a timestep, the ego may default to some backup behavior such as $\buctrlpol(\statenow)$ \cite{zhang2021safe}.

\section{Method}\label{sec:method}
We now show how to formally ensure safety with respect to occluded agents for mobile robots by computing the forward hidden set and danger zone, and verifying their disjointness with LPs.
For an agent of type $\is\in\{1,\cdots,\ntype\}$, the forward hidden set $\fhset_\is$ and the danger zone $\dz_\is$ can be computed with the following theorems:
\begin{thm}[Forward Hidden Set as AH-Polyhedrons]\label{thm:fhset_ahpoly}
    The forward hidden set $\fhset_\is$ for an adversary of type $\is$ is exactly a union of AH-polyhedrons, $\fhset_\is = \bigcup_{\ks=1}^{\nhset{\is}}\fhset_{\is,\ks}$, where: 
    \begin{subequations}
    \begin{align}
        \fhset_{\is,\ks} &= \ahpoly(\Aconfhset{\is, \ks}, \bconfhset{\is, \ks}, \Cconfhset{\is, \ks}, \dconfhset{\is, \ks}),\\
        \Aconfhset{\is, \ks} &= \begin{bmatrix}
            \Aconhset{\is, \ks} & \zeros \\
            \Aconadstate{\is}\adGdynstate{\is}\Cconhset{\is, \ks} & \Aconadstate{\is}\adGdynctrl{\is}\\
            \zeros & \Aconadctrl{\is}
        \end{bmatrix},\\
        \bconfhset{\is, \ks} &= \begin{bmatrix}
            \bconhset{\is,\ks}\\
            \bconadstate{\is}-\Aconadstate{\is}\adGdynstate{\is}\dconhset{\is, \ks} - \Aconadstate{\is}\adhdyn_\is\\
            \bconadctrl{\is}\\
        \end{bmatrix},\\
        \Cconfhset{\is, \ks} &= 
            \begin{bmatrix}\adGdynstate{\is}\Cconhset{\is, \ks} & \adGdynctrl{\is}\end{bmatrix},\\
        \dconfhset{\is, \ks} &= \adGdynstate{\is}\dconhset{\is, \ks} + \adhdyn_\is.
    \end{align}
    \end{subequations}
\end{thm}
\begin{proof}
    For each H-polyhedron $\hset_{\is,\ks}$ in the hidden set union, the corresponding forward hidden set $\fhset_{\is,\ks}$ can be written from \eqref{eq:fhset_def} and \Cref{assum:affine_dyn,assum:hpoly_sets} as $\fhset_{\is,\ks}=\{\adstatenext{\is}\ |\ \adstatenow{\is}=\Cconhset{\is, \ks}\yv+\dconhset{\is, \ks}, \Aconhset{\is, \ks}\yv\leq\bconhset{\is,\ks}, \adstatenext{\is}=\adGdyn_\is\begin{bmatrix}
        \adstatenow{\is}\tp&
        \adctrlnow{\is}\tp
    \end{bmatrix}\tp + \adhdyn_\is,\Aconadstate{\is}\adstatenext{\is}\leq\bconadstate{\is}, \Aconadctrl{\is}\adctrlnow{\is}\leq\bconadctrl{\is}\}$.
    Substituting $\adstatenow{\is}=\Cconhset{\is, \ks}\yv+\dconhset{\is, \ks}$ and $\adstatenext{\is}=\adGdyn_\is\begin{bmatrix}
        \adstatenow{\is}\tp &
        \adctrlnow{\is}\tp
    \end{bmatrix}\tp + \adhdyn_\is$ into $\fhset_{\is,\ks}$, we recover the AH-polyhedron above using $\begin{bmatrix}
        \yv\tp & \adctrlnow{\is}\tp
    \end{bmatrix}\tp$ as the states.
\end{proof}
\begin{thm}[Danger Zone as AH-Polyhedrons]\label{thm:dzone_ahpoly}
    Given a control signal $\ctrlnow$, the danger zone for an adversary of type $\is$ is exactly a union of AH-polyhedrons, $\dz_{\is} = \bigcup_{\js=1}^{\nstop(\ctrlnow)}\dz_{\is,\js}$, where:
    \begin{subequations}
    \begin{align}
        &\dz_{\is,\js} =  \ahpoly(\Acondz{\is, \js}, \bcondz{\is, \js}, \Ccondz{\is, \js}, \zeros),\\
        \begin{split}
        &\Acondz{\is,\js}= \\
        & {\setlength{\arraycolsep}{2.1pt}
        \begin{bmatrix}
            \Aconadvol{\is} & -\Bconadvol{\is}\adGdynstate{\is}^{\js-1} & -\Bconadvol{\is}\adGdynstate{\is}^{\js-2}\adGdynctrl{\is} & \cdots & -\Bconadvol{\is}\adGdynctrl{\is}\\
            \Aconvol(\bustate_{\ts + \js}) & \zeros & \zeros & \cdots & \zeros\\
            \zeros & \Aconadstate{\is} & \zeros & \cdots & \zeros\\
            \zeros & \Aconadstate{\is}\adGdynstate{\is} & \Aconadstate{\is}\adGdynctrl{\is} & \cdots & \zeros\\
            \vdots & \vdots & \vdots & \ddots & \vdots\\
            \zeros & \Aconadstate{\is}\adGdynstate{\is}^{\js-1} & \Aconadstate{\is}\adGdynstate{\is}^{\js-2}\adGdynctrl{\is} & \cdots & \Aconadstate{\is}\adGdynctrl{\is}\\
            \zeros & \zeros & \Aconadctrl{\is} & \cdots & \zeros\\
            \vdots & \vdots & \vdots & \ddots & \vdots\\
            \zeros & \zeros & \zeros & \cdots & \Aconadctrl{\is}
        \end{bmatrix},
        }
        \end{split}\\
        & \bcondz{\is, \js} = \begin{bmatrix}
            \bconadvol{\is} + \sum_{\ks=0}^{\js - 2}\Bconadvol{\is}\adGdynstate{\is}^{\ks}\adhdyn_\is\\
            \bconvol(\bustate_{\ts + \js})\\
            \bconadstate{\is}\\
            \bconadstate{\is} - \Aconadstate{\is}\adhdyn_\is\\
            \vdots\\
            \bconadstate{\is} - \sum_{\ks=0}^{\js-2}\Aconadstate{\is}\adGdynstate{\is}^{\ks}\adhdyn_\is\\
            \bconadctrl{\is}\\
            \vdots\\
            \bconadctrl{\is}
        \end{bmatrix},\\
        &\Ccondz{\is, \js} = \begin{bmatrix}
            \zeros_{\nadstate{\is}\times\nwspace} & \eye_{\nadstate{\is}} & \zeros
        \end{bmatrix}.
    \end{align}
    \end{subequations}
\end{thm}
\begin{proof}
    From \eqref{eq:dzone_def} and \Cref{assum:hpoly_sets,assum:affine_dyn}, each $\dz_{\is,\js}$ can be written as $\dz_{\is,\js}=\{\adstatenext{\is}\ |\ \Aconadvol{\is}\wspace\leq\Bconadvol{\is}\adstate_{\ts+\js,\is} + \bconadvol{\is}, \Aconvol(\bustate_{\ts + \js})\wspace\leq\bconvol(\bustate_{\ts + \js}), \adstate_{\ts+\ks+1,\is}=\adGdyn_{\is}\begin{bmatrix}
        \adstate_{\ts+\ks,\is}\tp&
        \adctrl_{\ts+\ks,\is}\tp
    \end{bmatrix}\tp + \adhdyn_{\is}, \Aconadstate{\is}\adstate_{\ts+\ks,\is}\leq\bconadstate{\is},\Aconadstate{\is}\adstate_{\ts+\js,\is}\leq\bconadstate{\is}, \Aconadctrl{\is}\adctrl_{\ts+\ks,\is}\leq\bconadctrl{\is},\ks=1,\cdots,\js-1\}$.
    We can rewrite $\adstate_{\ts+\ks+1,\is}=\adGdyn_{\is}\begin{bmatrix}
        \adstate_{\ts+\ks,\is}\tp&
        \adctrl_{\ts+\ks,\is}\tp
    \end{bmatrix}\tp + \adhdyn_{\is}$ as $\adstate_{\ts+\ks+1,\is}=\adGdynstate{\is}^{\ks}\adstate_{\ts+1,\is}+\adGdynstate{\is}^{\ks-1}\adGdynctrl{\is}\adctrl_{\ts+1,\is} + \adGdynstate{\is}^{\ks-2}\adGdynctrl{\is}\adctrl_{\ts+2,\is} + \cdots + \adGdynctrl{\is}\adctrl_{\ts+\ks,\is} + \adhdyn_\is + \adGdynstate{\is}\adhdyn_\is + \cdots + \adGdynstate{\is}^{\ks-1}\adhdyn_\is$ for $\ks=1,\cdots,\js-1$.
    Substituting into $\dz_{\is,\js}$, we recover the AH-polyhedron above using $\begin{bmatrix}
        \wspace\tp & \adstate_{\ts+1,\is}\tp & \adctrl_{\ts+1,\is}\tp & \cdots & \adctrl_{\ts+\js-1,\is}\tp
    \end{bmatrix}\tp$ as the states.
\end{proof}

The safety of $\ctrlnow$ can now be verified by solving LPs:
\begin{prop}[Occlusion-Aware Safety for Control Signal]\label{prop:occlusion_safe_lp}
    The control signal $\ctrlnow$ solves \Cref{prob:occlusion_control} if the LP
    \begin{align}\label{eq:occlusion_safe_lp}
    \begin{split}
        \regtext{find}\ & \xv,\\
        \suchthat & \Aconsafe{\is, \js, \ks} \xv \leq \bconsafe{\is, \js, \ks},\\
        & \Aconsafe{\is, \js, \ks} = \begin{bmatrix}
            \Aconfhset{\is, \ks} & \zeros \\
            \zeros & \Acondz{\is, \js} \\
            \Cconfhset{\is, \ks} & -\Ccondz{\is, \js}\\
            -\Cconfhset{\is, \ks} & \Ccondz{\is, \js}
        \end{bmatrix},\\
        & \bconsafe{\is, \js, \ks} = \begin{bmatrix}
            \bconfhset{\is, \ks} \\
            \bcondz{\is, \js}\\
            - \dconfhset{\is, \ks}\\
            \dconfhset{\is, \ks}
        \end{bmatrix},
    \end{split}
    \end{align}
    is infeasible for all $(\is, \ks) \in \{(1, 1), \cdots, (1, \nhset{1}), \cdots, (\ntype, 1), \cdots, (\ntype, \nhset{\ntype})\}$ and $\js \in \{1, \cdots, \nstop(\ctrlnow)\}$.
\end{prop}
\begin{proof}
    The LP \eqref{eq:occlusion_safe_lp} is exactly the emptiness check of the intersection between $\fhset_{\is,\ks}$ and $\dz_{\is,\js}$ \cite{sadraddini2019linear}, i.e.\ $\fhset_{\is,\ks}\cap\dz_{\is,\js}=\emptyset$, which certifies \Cref{prob:occlusion_control} for an $\is\in\{1,\cdots,\ntype\}$ and all $\js\in\{1, \cdots, \nstop(\ctrlnow)\}, \ks\in\{1, \cdots, \nhset{\is}\}$ as proven in \cite{zhang2021safe}.
    Since the ego's control policy is non-reactive, \eqref{eq:occlusion_safe_lp} certifies \Cref{prob:occlusion_control} for all $\is\in\{1,\cdots,\ntype\}$.
\end{proof}

By converting the game-theoretic safety conditions in \cite{zhang2021safe} into simple LPs, we enable formal verification in APRO without road-structure assumptions, overaproximations, or solving partial differential equations (PDEs) unlike existing set-based or value-function-based occlusion safety methods, providing a practical foundation to integrating occlusion-aware safety guarantees into mobile robots.

\section{Control Schemes}\label{sec:control_schemes}
Notice that \Cref{prop:occlusion_safe_lp} certifies occlusion-aware safety for a given control signal $\ctrlnow$.
Though this means that APRO can be used in bang-bang control against a backup policy, the resulting control signal is not necessarily optimal.
We provide two control schemes to address this: nonlinear optimization and bisection search.

\subsection{Nonlinear Optimization}\label{sec:opt_ctrl}
Consider an existing optimization problem where the cost is $\cost:\R^{\ncost}\to\R$ and the constraints are $\constraint:\R^{\ncost}\to\R^{\nconstraint}$.
If we assume that the preimage of any natural number through $\nstop$ is given as an H-polyhedron $\{\ctrl\in\ctrlset\ |\ \nstop(\ctrl) = \ms\} = \hpoly(\Aconstop{\ms}, \bconstop{\ms})$
for all $\ms \in \N$ (for example, intervals of acceleration may correspond to the number of timesteps needed to stop), then the optimization problem can be augmented with occlusion-aware safety as follows:
\begin{prop}[Safe Occlusion-Aware Optimal Control]\label{prop:occlusion_safe_opt}
    Let $\ms = \nstop(\ctrl)$.
    Then, the conditions in \Cref{prop:occlusion_safe_lp} hold for $\ctrl$ iff $\forall (\is, \ks) = (1, 1), \cdots, (\ntype, \nhset{\ntype}), \js = 1, \cdots, \ms$, $\exists \ctrl, \yv_{\is, \js, \ks}$ such that
    \begin{align}\label{eq:occlusion_safe_opt}
    \begin{split}
        \Aconstop{\ms}\ctrl &\leq \bconstop{\ms},\\
        \yv_{\is, \js, \ks} &\geq \zeros,\\
        \Aconsafe{\is, \js, \ks}\tp\yv_{\is, \js, \ks} &= \zeros,\\
        \bconsafe{\is, \js, \ks}\tp\yv_{\is, \js, \ks} &< 0.
    \end{split}
    \end{align}
\end{prop}
\begin{proof}
    Note that \eqref{eq:occlusion_safe_lp} is infeasible iff $\nexists \xv$ such that $\Aconsafe{\is, \js, \ks}\xv\leq\bconsafe{\is, \js, \ks}$.
    Thus, by Farkas' lemma \cite{matouvsek2007understanding}, \Cref{prop:occlusion_safe_lp} holds iff the constraints in \eqref{eq:occlusion_safe_opt} are fulfilled.
\end{proof}
\noindent \Cref{prop:occlusion_safe_opt} implies that the optimal solution to \Cref{prob:occlusion_control} is $\ctrlnow=\argmin_{\ctrl}\{\cost\ |\ \ctrl\in\{\optctrl_\ms\ |\ \ms\in\N\}\}$, where $\optctrl_\ms$ is the $\argmin_{\ctrl}$ of $\cost$ subjected to $\constraint$ and the constraints in \Cref{prop:occlusion_safe_opt}.
Note that in practice, $\bconsafe{\is, \js, \ks}\tp\yv_{\is, \js, \ks} < 0$ is treated as $\bconsafe{\is, \js, \ks}\tp\yv_{\is, \js, \ks} \leq -\eps$ for some small $\eps \in \R_+$.

We note that the constraints in \eqref{eq:occlusion_safe_opt} are nonlinear, which is solvable by general solvers such as Interior Point Opimizer (Ipopt) \cite{wachter2006implementation}.
If $\cost, \constraint$ are affine or bilinear, $\dyn, \buctrlpol, \bconvol$ are affine, and $\Aconvol$ is constant, then the optimization is specifically a bilinear program that can be handled by more specialized solvers such as Gurobi \cite{gurobi2023gurobi}.
Furthermore, \eqref{eq:occlusion_safe_opt} only needs to be solved for all $\ms$ whose preimage is non-empty, which is typically finite when $\ctrlset$ is bounded.
This allows APRO to augment existing optimization-based planners such as \cite{zhang2020optimization} with occlusion-aware safety guarantees.
Nevertheless, solving \eqref{eq:occlusion_safe_opt} may still be computationally expensive, motivating us to next present a faster control scheme suitable for online planning.

\subsection{Bisection Search}\label{sec:bisection_search}
Instead of slowly solving a nonlinear program, by assuming $\ctrlset \subseteq \R$, $\cost$ is monotonically decreasing with $\ctrlnow$, and safety is also monotonic (i.e.\ there exists a known $\ctrlsafelb \in \ctrlset$ such that for any $\ctrlsafeub \in \ctrlset$ that solves \Cref{prob:occlusion_control}, any $\ctrlnow \in [\ctrlsafelb, \ctrlsafeub]$ is also a solution), we can obtain a near-optimal control signal by evaluating \Cref{prop:occlusion_safe_lp} $\nbisect \in \N$ times using a standard bisection search over $[\ctrlmin, \ctrlmax]$, where $\ctrlmax \in \ctrlset$ is an upper bound signal and $\ctrlmin \in [\ctrlsafelb, \ctrlmax]$ is a backup signal.

Many works in occlusion-aware planning uses path-velocity decomposition \cite{lee2021limited}, in which a high-level planner first defines a path and a lower-level controller regulates acceleration along it.
In this case, the required assumptions are fulfilled since we only control acceleration along the path, going faster is preferred, and going slower is always safer.
This enables APRO to augment existing path planners \cite{nawaz2025occupancy, chung2026selecting} with real-time occlusion-aware safety guarantees.
\section{Simulation Experiments}\label{sec:sim_exp}
We now validate APRO with simulation experiments by comparing the control schemes in \Cref{sec:control_schemes} and existing methods with formal safety guarantees.
All simulation experiments were performed using Python on a desktop computer with an AMD Ryzen Threadripper 7960X processor (24 cores, 48 threads) and 125 GiB RAM.

\subsection{Experiment Setup}
We conducted two simulation experiments: a controlled setup through a narrow gap of hidden set, and a more realistic setup within a data replay of a real-life dataset.
In both experiments, the ego follows double integrator dynamics under a path-velocity decomposition scheme \cite{lee2021limited} with velocity $\in [0, 2]$ and acceleration $\in [-2, 2]$.
When it observes an occluded agent, its backup policy $\buctrlpol$ is to brake until it stops.
The ego occupies a rectangular volume of length $4.6$ and width $1.85$, rotated by its current heading.

At each timestep for all methods, the ego executes the control signal returned by the method, or the backup policy if no signal is returned.
We verify safety of the proposed control by checking if the FRS of the hidden set with adversary dynamics intersects with the ego's volume if the ego were to brake to a stop, which can be computed in closed-form with Minkowski sum in this specific setup.

\textit{Metrics:}
We report safety rate as the ratio of safe timesteps over all timesteps where the backup policy is not triggered. 
We also reported the mean computation time per timestep, defined as the time taken from obtaining the hidden set to returning the control signal.

\subsubsection{Narrow Gap Experiment}\label{sec:gap_exp}
We start off with a controlled setup to minimize sources of confounding variables and isolate the effects of our control scheme and reachable set representation.
In this experiment, the ego is initialized before two rectangular blocks of hidden set that forms a narrow gap of width $\gapsize\in\{4, 5, 6, 7\}$.
The hidden set corresponds to occluded pedestrians with 2-D double integrator dynamics with velocity $\in[-1.2, 1.2]^2$ and acceleration $\in[-0.5, 0.5]^2$.
The pedestrians occupy no volume.
We recorded the number of timesteps ($\dt=0.4$) it takes for the ego to pass through the gap, or declare timeout when $150$ timesteps have been reached.

To compare the effects of the different control schemes in \Cref{sec:control_schemes}, we compare APRO implemented on bang-bang control, bisection search (\Cref{sec:bisection_search}) with $\nbisect=8$, and optimal control (\Cref{sec:control_schemes}).
For optimal control, we compared the effects of using a bilinear program solver \cite{gurobi2023gurobi} and that of using a general nonlinear program solver \cite{wachter2006implementation}.
We set the cost as $-\accelsnow$ to encourage higher speed and apply no additional constraints beyond those related to occlusions.

\textit{Baselines:}
We compare against other occlusion-aware planning methods with safety guarantees by incorporating them in a bang-bang control scheme akin to APRO's.
We implement \cite{lee2021limited} and \cite{firoozi2022occlusion} by setting each edge of the hidden set as an occluded edge, and overapproximating the ego's volume with a circle in \cite{firoozi2022occlusion}.
We also compare against \cite{zhang2021improved} by computing the danger zone and forward hidden set with an HJ reachability toolbox \cite{schmerling2021hj} (since their analytic solution does not apply to our experiment setup).
For the danger zone, we uniformly divided the states $[-5, 5]^2\times[-1.2, 1.2]^2$ relative to the ego's position into $\ngrid^4$ grids, where $\ngrid\in\{30, 50\}$.
For the forward hidden set, we similarly divided $[10, 30]\times[14-0.5\gapsize, 16+0.5\gapsize]\times[-1.2, 1.2]^2$ into $\ngrid^4$ grids.
We computed the sets using PDE solvers with low and very high accuracy \cite{schmerling2021hj}, and check their intersections by interpolating the resulting value functions.
As a control baseline, we also reported the performance when the ego always travels at maximum speed.
An overview of the different baselines is shown in \Cref{fig:narrow_gap}.

\begin{figure}[t]
\centering
    \includegraphics[width=1\columnwidth]{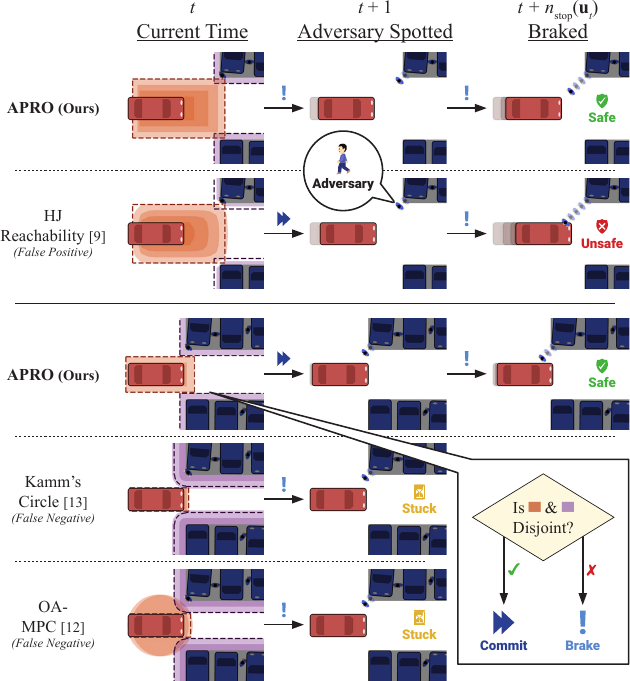}
\caption{
An overview of the different methods compared in \Cref{sec:gap_exp}.
For each method, the ego (red) commits to the proposed control if the orange set (danger zone in APRO and \cite{zhang2021safe}) is disjoint from the purple set (forward hidden set in APRO and \cite{zhang2021safe}), whose ground truths are labeled with dashes.
APRO is the only one that is \textit{exact}: the ego is not conservative, but is able to brake in time if an agent comes out of occlusion (gray rectangles).
}\label{fig:narrow_gap}
\vspace*{-0.5cm}
\end{figure}

\subsubsection{AVP Experiment}\label{sec:avp_demo}
We now validate APRO on a more realistic scenario: a data replay of a crowded parking lot from the Dragon Lake Parking (DLP) Dataset \cite{shen2022parkpredict+}, involving multiple types of occluded agents, dynamic obstacles, changing ego plans, and construction of hidden sets from obstacle representations.
We compare against other methods \cite{zhang2021improved,lee2021limited} implemented as safety filters on top of a Hybrid A* planner \cite{chung2026selecting, nawaz2025occupancy}.
At each initialization, we randomly place the ego at the start or end, and top, middle, or bottom of each of the $4$ horizontal lanes in a randomly chosen scene of the dataset, then use \cite{nawaz2025occupancy, chung2026selecting} to generate a path for the ego to follow at every $20$ timesteps ($\dt=0.2$).
At every timestep, the ego perceives all obstacles intersecting with a rectangular (width $12.025$, height $13.8$) field of view aligned with the back of the car \cite{nawaz2025occupancy}.
We then obtain the hidden set by greedily merging the obstacles with convex hull (until $6$ remains), raycasting from the ego to obtain the occluded area, intersecting with where the adversaries are expected to appear, and performing Pontryagin difference with their volumes \cite{gillies2024shapely}, and Cartesian product with their non-workspace domain.
We visualize this in \Cref{fig:avp}.
The occluded agents include pedestrians with the same specifications as \Cref{sec:gap_exp}, as well as horizontally and vertically traveling cars with the same specifications as the ego and assumed to appear only in their corresponding lanes.
The simulation terminates after $100$ timesteps and we record the average speed of each simulation.

We run $30$ randomized scenarios each with APRO, \cite{lee2021limited} (augmented for volumetric adversaries using \cite{koschi2017spot}), the no safety control, and \cite{zhang2021safe} (with low solver accuracy and $\ngrid=30$, and very high solver accuracy and $\ngrid=50$, where the grids divide the Cartesian product of $[-20, 20]^2$ and the adversaries' non-workspace domain, relative to the ego's current position) on the bisection search scheme with $\nbisect=4$.
We did not compare with \cite{firoozi2022occlusion} due to its inability to account for occluded agents with volume.

\begin{figure}[t]
\centering
    \includegraphics[width=1\columnwidth]{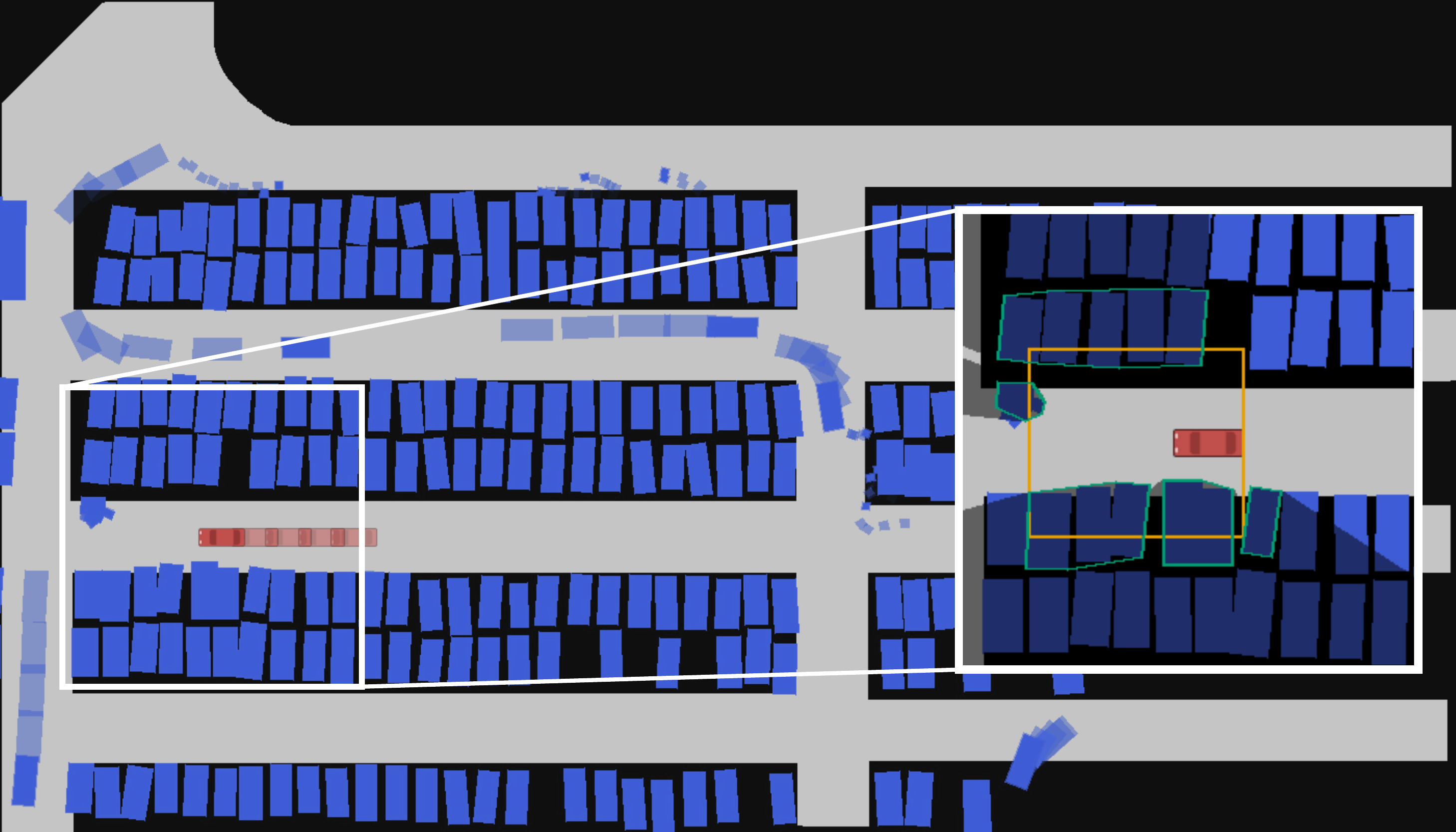}
\caption{
A snapshot of our AVP experiment where the ego (red) navigates a data replay of the DLP Dataset \cite{shen2022parkpredict+} with dynamic obstacles (blue).
At each timestep, the ego perceives obstacles within its field of view (orange), merge the obstacles with convex hull (green), and use raycasting to obtain the occluded area (shaded).
APRO enables safety against all possible adversaries hiding within the occluded area without being too conservative.
}\label{fig:avp}
\vspace*{-0.5cm}
\end{figure}

\subsection{Results and Discussion}
See results in \Cref{table:results}.
APRO maintains a $100\%$ safety guarantee against occluded agents regardless of the control scheme used, confirming the validity of our formulations.
Among all control schemes, bang-bang control has the lowest computation time, though it can return suboptimal solutions.
Compared to using a general solver \cite{wachter2006implementation} to find the optimal solution, \cite{gurobi2023gurobi} reduces the computation time by $\approx2$ to $13$ times, which is further reduced by $\approx4$ to $9$ times using bisection search.
However, we note that these control schemes require different sets of assumptions, so the slower methods may still be useful depending on the application.

\begin{table*}[!ht]
\centering
\captionsetup{font=small}
\footnotesize
\caption{Results of our proposed method, APRO, under different control schemes, compared with other occlusion-aware safety methods with formal guarantees \cite{zhang2021safe, lee2021limited, firoozi2022occlusion} in the narrow gap and AVP experiments.
The best results in each evaluation metric ($\uparrow$: higher is better, $\downarrow$: lower is better), sans those from not having a safety control, are bolded.}
\label{table:results}
\begin{tabular}{c|r||r|r|r|r||r|r|r|r||r||r||r}
    &
    \multirow{6}{*}{%
  \parbox{0.483cm}{\centering
    Gap\\
    Size\\
    ($\gapsize$)
  }%
}
& \multicolumn{11}{c}{Method}\\
     \cline{3-13}
     && \multicolumn{4}{c||}{APRO (Ours)} & \multicolumn{4}{c||}{HJ Reachability \cite{zhang2021safe}} & & \\
     \cline{3-10}
    Scena-&  & & & & & \multicolumn{4}{c||}{Solver Accuracy} & \multicolumn{1}{c||}{Kamm's} & \multicolumn{1}{c||}{OA-}\\
    \cline{7-10}
    rio&  & \multicolumn{1}{c|}{Bang-} & \multicolumn{1}{c|}{Ipopt}&\multicolumn{1}{c|}{Gurobi} & \multicolumn{1}{c||}{Bisec-}& \multicolumn{2}{c|}{Low} &\multicolumn{2}{c||}{Very High} & \multicolumn{1}{c||}{Circle} & \multicolumn{1}{c||}{MPC} & \multicolumn{1}{c}{None}\\
    \cline{7-10}
    & &\multicolumn{1}{c|}{Bang}&\multicolumn{1}{c|}{\cite{wachter2006implementation}}&\multicolumn{1}{c|}{\cite{gurobi2023gurobi}}&\multicolumn{1}{c||}{tion}&\multicolumn{4}{c||}{Grid Size ($\ngrid$)}&\multicolumn{1}{c||}{\cite{lee2021limited}}&\multicolumn{1}{c||}{\cite{firoozi2022occlusion}}\\
    \cline{7-10}
    &&&&&& 30 & 50 & 30 & 50 &&\\
    \hline\hline
    \multicolumn{13}{c}{Timesteps Needed to Reach the Goal $\downarrow$}\\
    \hline\hline
    &4&133&\textbf{78}&\textbf{78}&\textbf{78}&123&128&130&131&Timeout&Timeout&46\\
    \cline{2-13}
    Narrow&5&80&\textbf{51}&\textbf{51}&\textbf{51}&74&76&77&78&128&Timeout&46\\
    \cline{2-13}
    Gap&6&\textbf{46}&\textbf{46}&\textbf{46}&\textbf{46}&\textbf{46}&\textbf{46}&\textbf{46}&\textbf{46}&78&Timeout&46\\
    \cline{2-13}
    &7&\textbf{46}&\textbf{46}&\textbf{46}&\textbf{46}&\textbf{46}&\textbf{46}&\textbf{46}&\textbf{46}&\textbf{46}&114&46\\
    \hline\hline
    \multicolumn{13}{c}{Average Speed (ms$^{-1}$) $\uparrow$}\\
    \hline\hline
    \multirow{2}{*}{AVP}&\multirow{2}{*}{-}&\multirow{2}{*}{-}&\multirow{2}{*}{-}&\multirow{2}{*}{-} & 0.846 & \textbf{1.219} &\multirow{2}{*}{-}&\multirow{2}{*}{-}&1.017&0.692 &\multirow{2}{*}{-} & 1.224\\
    &&&&&$\pm$0.822&$\pm$\textbf{0.939}&&&$\pm$0.884&$\pm$0.710&&$\pm$0.947\\
    \hline\hline
    \multicolumn{13}{c}{Safety Rate $\uparrow$}\\
    \hline\hline
    &4&\textbf{1.000}&\textbf{1.000}&\textbf{1.000}&\textbf{1.000}&0.945&0.973&0.987&0.987&\textbf{1.000}&\textbf{1.000}&0.500\\
    \cline{2-13}
    Narrow&5&\textbf{1.000}&\textbf{1.000}&\textbf{1.000}&\textbf{1.000}&0.920&0.960&0.960&0.980&\textbf{1.000}&\textbf{1.000}&0.500\\
    \cline{2-13}
    Gap&6&\textbf{1.000}&\textbf{1.000}&\textbf{1.000}&\textbf{1.000}&\textbf{1.000}&\textbf{1.000}&\textbf{1.000}&\textbf{1.000}&\textbf{1.000}&\textbf{1.000}&1.000\\
    \cline{2-13}
    &7&\textbf{1.000}&\textbf{1.000}&\textbf{1.000}&\textbf{1.000}&\textbf{1.000}&\textbf{1.000}&\textbf{1.000}&\textbf{1.000}&\textbf{1.000}&\textbf{1.000}&1.000\\
    \hline
    \multirow{2}{*}{AVP}&\multirow{2}{*}{-}&\multirow{2}{*}{-}&\multirow{2}{*}{-}&\multirow{2}{*}{-}&\textbf{1.000}&0.551&\multirow{2}{*}{-}&\multirow{2}{*}{-}&0.544&\textbf{1.000}&\multirow{2}{*}{-}&0.548\\
    &&&&&\textbf{$\pm$0.000}&$\pm$0.332&&&$\pm$0.344&\textbf{$\pm$0.000}&&$\pm$0.340\\
        \hline\hline
    \multicolumn{13}{c}{Computation Time per Timestep (s) $\downarrow$}\\
    \hline\hline
    &\multirow{2}{*}{4}&0.003&0.347&0.063&0.016&1.227&1.422&4.948&5.927&\textbf{0.001}&0.002&0.000\\
    &&$\pm$0.001&$\pm$0.389&$\pm$0.013&$\pm$0.007&$\pm$0.039&$\pm$0.077&$\pm$0.201&$\pm$0.267&$\pm$\textbf{0.000}&$\pm$0.001&$\pm$0.000\\
    \cline{2-13}
    &\multirow{2}{*}{5}&0.005&1.035&0.077&0.019&1.281&1.469&4.995&6.051&\textbf{0.001}&0.003&0.000\\
    Narrow&&$\pm$0.001&$\pm$2.338&$\pm$0.022&$\pm$0.012&$\pm$0.256&$\pm$0.020&$\pm$0.290&$\pm$0.351&$\pm$\textbf{0.000}&$\pm$0.001&$\pm$0.000\\
    \cline{2-13}
    Gap&\multirow{2}{*}{6}&0.006&0.113&0.052&0.006&1.370&1.476&4.977&6.151&\textbf{0.002}&0.003&0.000\\
    &&$\pm$0.000&$\pm$0.021&$\pm$0.003&$\pm$0.000&$\pm$0.508&$\pm$0.029&$\pm$0.065&$\pm$0.547&$\pm$\textbf{0.000}&$\pm$0.001&$\pm$0.000\\
    \cline{2-13}
    &\multirow{2}{*}{7}&0.006&0.104&0.053&0.006&1.306&1.522&4.995&6.193&\textbf{0.002}&0.005&0.000\\
    &&$\pm$0.000&$\pm$0.022&$\pm$0.004&$\pm$0.000&$\pm$0.010&$\pm$0.010&$\pm$0.057&$\pm$0.694&$\pm$\textbf{0.000}&$\pm$0.001&$\pm$0.000\\
    \hline
    \multirow{2}{*}{AVP} & \multirow{2}{*}{-} & \multirow{2}{*}{-} & \multirow{2}{*}{-} & \multirow{2}{*}{-} & 0.069 & 2.959 & \multirow{2}{*}{-} & \multirow{2}{*}{-} & 32.145 & \textbf{0.039} & \multirow{2}{*}{-} & 0.000\\
    & & & & & $\pm$0.049 & $\pm$1.583 & & & $\pm$24.135 & $\pm$\textbf{0.021} & & $\pm$0.000
\end{tabular}
\end{table*}

Though there are instances when \cite{zhang2021safe} enables the ego to travel faster and arrive at the goal earlier than APRO, their proposed controls are \textit{not} actually safe, even at the highest accuracy setting with the highest grid size (setting $\ngrid=60$ results in memory error in Python).
We attribute this to numerical errors in gridding and PDE-solving, which are known issues for HJ reachability \cite{he2025threshold}.
Regardless, all settings of HJ require computation time far slower than what is typically required for real-life deployment.
On the other hand, \cite{lee2021limited, firoozi2022occlusion} are able to maintain their $100\%$ safety guarantees, with similar or even lower computation time compared with APRO.
However, their conservatism can cause them to get stuck before the narrow gap (timeout) or yield suboptimal solutions (longer timesteps needed to reach the goal and slower average speed).
We also stress that \cite{lee2021limited} is only able to handle double integrator dynamics and \cite{firoozi2022occlusion} uses knowledge of only the adversaries' max velocity, whereas APRO is flexible to handling multiple types of agents with different affine dynamics.

\section{Hardware Experiment}\label{sec:hardware}
To demonstrate the applicability of APRO on a real robot, we recreated the experiment setup in \Cref{sec:gap_exp} with an F1/10 class robotic vehicle \cite{o2019f1} equipped with a Hokuyo UST-10LX lidar sensor.
Our method is computed with Python and ROS 2 \cite{macenski2022robot} on the model car's on-board computer, which has an AMD Ryzen AI 9 HX PRO 370 processor (12 cores, 24 threads) and 27 GiB RAM.

\subsection{Demonstration Setup}
Similar to \Cref{sec:sim_exp}, we placed foam boards to create a narrow corridor of hidden set with gap size $\gapsize\in\{0.65 \regtext{ m}, 0.73 \regtext{ m}, 0.81 \regtext{ m}\}$ around the ego.
At each timestep, the ego constructs the polyhedral hidden set by perceiving its distance from the wall using the lidar sensor.
It then computes a maximum safe velocity to travel at using APRO (bisection search, $\nbisect=4$) by modeling itself as a double integrator with velocity $\in[0\regtext{ ms}^{-1}, 3\regtext{ ms}^{-1}]$ and acceleration $\in[-2\regtext{ ms}^{-2}, 2\regtext{ ms}^{-2}]$, and the pedestrian adversary as a double integrator with velocity $\in[-0.2\regtext{ ms}^{-1}, 0.2\regtext{ ms}^{-1}]^2$ and acceleration $\in[-0.1\regtext{ ms}^{-2}, 0.1\regtext{ ms}^{-2}]^2$.
After accelerating to the maximum safe velocity, we put out an obstacle to trigger the lidar sensor, after which the ego would start to brake.
For evaluation, we measure the computation time per timestep (from constructing the hidden set to returning the safe velocity) and the safety rate (the ego is safe if a pedestrian with the assumed specifications cannot theoretically reach the ego from the hidden set within the braking time).
For each gap size, we repeat the test for $5$ times, then repeat the test for $5$ more times without using our method (traveling at maximum speed).
A picture of our setup is shown in \Cref{fig:front_figure}.

\subsection{Results and Discussion}
In all runs, APRO is always safe with a safety rate of $1.000\pm0.000$, whereas the baseline always fails with a safety rate of $0.000\pm0.000$.
Even running on the on-board computer, APRO still has a computation time of $0.040\pm0.013$ s per timestep, which is faster than the $10$ Hz typically desired for practical deployment.

We note that model mismatch did occur when translating to hardware.
For example, the ego's maximum deceleration was nonlinear, and the ego's states can get out of the expected bounds.
In this case, we were able to maintain safety due to underapproximation of the ego's maximum deceleration (which overapproximates the danger zone).
To more formally and robustly maintain guarantee under model mismatch, error bounds can be integrated into APRO if known, similar to \cite{chung2024goal, chung2025guaranteed}, which we leave as future work.

\section{Conclusion}\label{sec:conclusion}
In this paper, we present APRO, an occlusion-aware planning framework for mobile robot navigation that leverages game-theoretic safety conditions and exact AH-polyhedron reachability analysis.
By reformulating the forward hidden set and danger zone computations in closed-form using AH-polyhedrons, APRO provides formal safety guarantees without conservatism, numerical errors, and scalability limitations of prior set-based and HJ reachability approaches.
The effectiveness of the approach was validated through simulation and hardware experiments, including a real-world-inspired AVP data replay, where APRO consistently achieved a $100\%$ safety rate while maintaining real-time performance and compatibility with real sensors such as lidar and existing planners such as Hybrid A*.

\subsubsection*{Limitations}
We observe three major limitations for APRO.
Firstly, APRO has only been demonstrated under idealized modeling assumptions, including perfect affine dynamics, perfect perception, and exact knowledge of the hidden set.
This restricts applicability to robots and adversaries with more complex or highly nonlinear behaviors and settings with significant uncertainty.
Though APRO can theoretically be extended to PWA or learned systems with bounded modeling and perception error bounds using strategies similar to \cite{chung2024goal, chung2025guaranteed}, we leave it as future work.
Secondly, the computation time of APRO scales linearly with the number of obstacles and adversary types.
This may possibly be improved by solving LPs in parallel or encoding multiple LPs into a mixed-integer linear program (MILP) \cite{chung2025provably}, which we also leave to future work.
Finally, we acknowledge that our hardware experiment is only performed on a simplistic scenario on an F1/10 model car.
Towards this, we are working on incorporating APRO on a full-scale AVP system navigating a real parking lot with camera-based sensors.





\renewcommand{\bibfont}{\normalfont\footnotesize}
{\renewcommand{\markboth}[2]{}
\printbibliography}

\end{document}